\documentclass[11pt]{article}
\usepackage[preprint]{acl}
\usepackage{times,latexsym}
\usepackage[T1]{fontenc}
\usepackage[utf8]{inputenc}
\usepackage{microtype,graphicx,amsmath,amssymb,mathtools,booktabs,xcolor,enumitem,array,float}
\usepackage{algorithm}
\usepackage[noend]{algpseudocode}
\usepackage{amsthm}
\newtheorem{theorem}{Theorem}
\newtheorem{proposition}{Proposition}
\newcommand{\guided}{\textsc{Guided-AP}}
\newcommand{\probecal}{\textsc{ProbeCal}}
\newcommand{\atp}{AtP}
\newcommand{\atpstar}{AtP$^{*}$}

\newcommand{\E}{\mathbb{E}}
\newcommand{\Var}{\operatorname{Var}}

\title{Correcting Attribution Patching with a Few Exact Measurements}
\author{Luyang Zhang$^{1}$\thanks{Equal contribution.} \quad Minxing Zheng$^{1}$\footnotemark[1] \quad Jialu Wang$^{2}$ \\
  $^{1}$Carnegie Mellon University \quad $^{2}$University of California, Santa Cruz}
\begin{document}
\maketitle
\begin{abstract}
Mechanistic interpretability asks which language model components matter for a behavior, and activation patching answers by changing one component's activation and measuring its effect, at a cost growing with components and prompts. Attribution patching estimates these effects from one gradient pass, but its estimates can be far off, and verifying each one restores the cost. Can the structure of these errors reduce this cost? We propose \guided{}, which uses a few exact measurements to reveal attribution patching's errors and correct unmeasured estimates. On a single prompt, many components' errors share one sign, so a small exact sample yields one correction for the rest. Over many prompts of one task, components are misestimated similarly, so a few measured prompts show what to verify and how to correct each component's average effect. Our analysis shows that choosing what to verify needs error size, while correction also needs sign. On a single prompt, the correction lowers error by up to 26\% across three behaviors at the compute of verifying the largest estimates and also on larger models. Over many prompts, learned corrections cut the error of average effects on factual completion by 75\% to 79\%, repaying their cost in long analyses.
\end{abstract}

\section{Introduction}
Mechanistic interpretability seeks to explain how a language model's internal computations produce its behavior. To study behaviors such as resolving a reference or recalling a fact, researchers measure how internal components, such as attention heads or MLP neurons, influence the model's prediction. One widely used approach for measuring this influence is activation patching (AP) \citep{wang2023ioi,zhang2024patching}. It replaces a component's activation with one from a contrasting input and measures the resulting change in a prediction score, such as the logit difference between two answers.

\begin{figure*}[t]
\centering
\includegraphics[width=\textwidth]{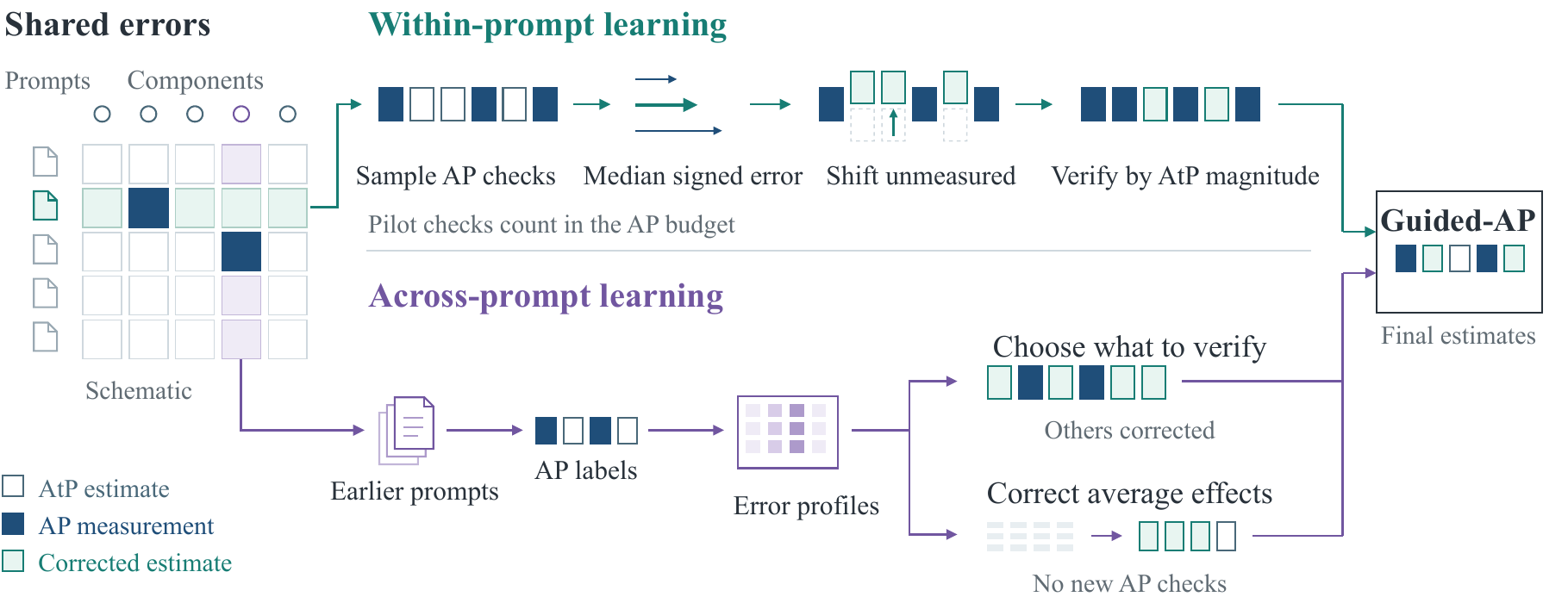}
\caption{\guided{} learns shared attribution errors. Within a prompt, a small exact sample corrects unmeasured estimates; across prompts, learned error profiles choose what to verify and correct remaining estimates and average effects. Final estimates are exact where verified and corrected elsewhere.}
\label{fig:shared-errors}
\end{figure*}

However, AP is expensive across many components and prompts, because each measurement is a separate intervention. To reduce this cost, attribution patching (\atp{}) approximates the same effects using gradients from a single backward pass \citep{nanda2023atp,syed2024eap}. This efficiency comes at the cost of accuracy, because its estimates can be far from the measured effects \citep{zhang2026attribution}. Verifying every estimate with AP would remove the error but restore the original cost, so the challenge is to improve the estimates without exhaustive measurement. Under a limited budget, verifying the components with the largest estimated effects is a strong baseline, but a large effect does not mean a large error. Errors can concentrate in a few components and recur across prompts, as illustrated in Figure~\ref{fig:motivation}. Errors can also persist when estimates are averaged across prompts. The central question of this paper is whether knowing the structure of these errors makes attribution analysis more effective and more efficient.

Our answer builds on the fact that an exact measurement reveals both a component's true effect and the error of its cheap estimate. When such errors are similar across estimates, a few measurements can repair estimates that were never measured. Each measurement therefore improves more than one estimate. It replaces its own estimate with the true effect, and the error it reveals improves the estimates left unmeasured. Each error belongs to one component on one prompt, so we can look at the errors from two perspectives (Figure~\ref{fig:shared-errors}). Our method, \guided{}, applies the same estimator in both settings. On a single prompt, the errors of many components share one sign, so one shift improves many estimates at once. We measure a small, representative sample and shift every unmeasured estimate by the median of the sampled
errors, the shift with the lowest mean absolute error on the sample. The rest of the budget verifies the largest estimated effects. Over many prompts, each component's error tends to recur. A few measured prompts give each component an error profile. Later prompts reuse it to choose verifications and correct each component's average effect. \guided{} learns from exact AP measurements.

On a single prompt, the correction needs summary statistics of the error distribution (e.g., its median), and a small exact sample on that prompt estimates them. Centering estimates without measuring their errors does not reproduce this gain. With the exact sample, the correction lowers error by up to 26\% at the same compute. Over many prompts, the errors left after averaging mostly share one sign, so a few calibration prompts can estimate each component's average error. A correction learned from exact measurements cuts the error of these averages on factual completion by 75\% to 79\% relative to uncorrected averages, with no measurement on the analyzed prompts. Using the same measurements directly as average effects is 6 to 15 times worse, so the gain comes from modeling the errors. Choosing what to verify needs only the size of each error. In our cost projections, the correction pays for itself within a few thousand prompts, and learned selection needs longer analyses, up to tens of thousands of prompts. The single-prompt gains carry to larger models and finer components, up to a 32B model with tens of thousands of components per layer, and on some tasks the correction reaches a fixed accuracy target several times faster than verification alone.

We make three contributions.
\begin{itemize}[leftmargin=*,itemsep=2pt,topsep=3pt]
\item \textbf{A framework of budgeted attribution analysis.} We study attribution analysis under a fixed budget of exact measurements that must serve both verification and correction, and we compare methods by the error left at equal compute.

\item \textbf{Analysis of error structures.} We show that attribution errors are shared within a single prompt and recur over many prompts, and which error statistic selection and correction each need.

\item \textbf{Scalability to finer components and larger models.} We show that a fixed budget of exact measurements remains useful as models grow larger and their components finer, where each measurement must go further.

\end{itemize}

\section{Related Work}

\paragraph{Causal attribution and circuit discovery.}
Activation patching localizes behavior through interventions whose results depend on corruption, component definitions, and output metrics \citep{wang2023ioi,zhang2024patching}. Automated circuit discovery reduces search costs \citep{conmy2023acdc,syed2024eap}; benchmarks and a shared task evaluate localization across tasks \citep{mueller2025mib,arad2025blackboxnlp,hanna2024faith}, and feature-level methods reach tens of thousands of components \citep{marks2025sparse}. We estimate isolated component effects and their dataset means; low effect-vector error does not imply faithfulness of a jointly intervened circuit.

\paragraph{Correcting attribution estimates.}
Attribution patching linearizes the intervention \citep{nanda2023atp}. Existing methods correct estimates within a prompt. \atpstar{} targets specific failure modes \citep{kramar2024atpstar}, relevance patching replaces gradients with propagation coefficients \citep{jafari2025relp}, interaction-aware backpropagation fixes gradient flow through softmax and normalization \citep{edin2025gim}, path-based attribution integrates along the intervention \citep{sundararajan2017ig,janizek2021integratedhessians}, and second-order patching correction expands effects at the endpoint \citep{zhang2026attribution}. We learn error structure from exact measurements and reuse it across components within a prompt or across prompts for a component. Our base estimate is \atp{}.

\paragraph{Efficient error estimation.}
Randomized diagonal estimators approximate the diagonal of a matrix from matrix-vector products \citep{hutchinson1989trace,bekas2007diagonal,baston2022diagpp,epperly2024xtrace}, and Hessian-vector products make this cheap for neural networks \citep{pearlmutter1994hvp,yao2021adahessian}; we use such sketches only as a comparison label source. Allocating exact measurements to unreliable estimates follows active learning \citep{cohn1994active,lewis1994sequential} and selective prediction \citep{geifman2017selective}; influence estimation at scale similarly uses cheap curvature approximations \citep{koh2017influence,grosse2023influence}. We test whether reusing attribution errors within or across prompts improves accuracy after charging acquisition, fitting, and verification costs.

\section{Proposed Method}
\label{sec:method}

\subsection{Preliminaries}
\label{sec:preliminaries}

Consider a language model with $N$ candidate components, a clean prompt $x$, and a contrasting prompt $x'$. Write $a_i$ and $a_i'$ for component $i$'s activations on the two prompts and $\delta_i=a_i'-a_i$ for its patch direction. These activations may be scalars or vectors. Let $F(t)$, for $t\in\mathbb R^N$, be the prediction metric on $x$ when each component receives the additive shift $t_i\delta_i$ and later layers are recomputed; $F(0)$ is the clean run.

To isolate the effect of each component, let $v_i\in\mathbb R^N$ denote the coordinate unit vector with $i$th entry $1$ and $0$ elsewhere, and define the isolated AP effect of component $i$, denoted $\Delta_i$, and its \atp{} approximation $g_i$ as
\begin{equation*}
\Delta_i=F(v_i)-F(0),\qquad g_i=\partial_iF(0).
\end{equation*}
We call $e_i=\Delta_i-g_i$ the \atp{} residual, and evaluate an estimate $\widehat\Delta$ of $(\Delta_1,\ldots,\Delta_N)$ by its mean absolute error, $\operatorname{MAE}(\widehat\Delta)=N^{-1}\sum_{i=1}^N|\Delta_i-\widehat\Delta_i|$.

Though accurate, exact AP requires $N$ separate interventions per prompt, with a cost linear in both the number of components and the number of prompts. In comparison, \atp{} estimates all $N$ effects from one backward pass, but its linear approximation misses the curvature of a finite patch, so its residuals can be large \citep{zhang2026attribution}.

\subsection{Learning Shared Attribution Errors}
\label{sec:framework}
Given these two limitations, we seek a framework that balances estimation accuracy and computation. Our key observation is that \atp{} residuals follow structured patterns. First, a few components carry most of the error. Second, on a single prompt, the residuals are biased away from zero. Third, each component's residual is consistent across prompts of the same task. Figure~\ref{fig:motivation} demonstrates these patterns with a toy example on GPT-2.

\begin{figure}[t]
\centering
\includegraphics[width=\columnwidth]{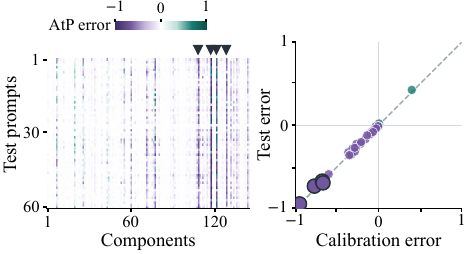}
\caption{\atp{} errors on GPT-2. (a) The top four components ($\blacktriangledown$) carry 80\% of the test error. (b) Calibration versus test error for each component, with the top four highlighted. The two errors are strongly correlated and consistently biased away from zero. }
\label{fig:motivation}
\end{figure}

Motivated by these patterns, we propose to measure only a few well-chosen components exactly, which removes most of the \atp{} error, and to use their residuals to correct the remaining unmeasured components. Specifically, given a budget of $B$ exact AP measurements per prompt, we aim to find a subset $S$ of the $N$ components with $|S|=B$ and construct the estimator
\begin{equation}
\widehat\Delta_i=
\begin{cases}
\Delta_i,&i\in S,\\
g_i+c_i,&i\notin S,
\end{cases}
\label{eq:estimator}
\end{equation}
where $c_i$ is learned from exact measurements. Since measured components are exact, choosing $S$ and $c$ to minimize the MAE of $\widehat\Delta$ gives
\begin{equation}
\begin{aligned}
\min_{S,\,c}\;\; & \sum_{i=1}^{N}|\widehat\Delta_i-\Delta_i|=\sum_{i\notin S}|g_i+c_i-\Delta_i|\\
\text{s.t.}\;\; & S\subseteq\{1,\ldots,N\},\ |S|=B.
\end{aligned}
\label{eq:utility}
\end{equation}
Equation~\ref{eq:utility} requires the unknown residuals $\Delta_i-g_i$ for $i\notin S$. The patterns above, however, can make these residuals predictable from measured ones. Within a prompt, measured components reveal the shared offset of the rest; across prompts, a component's residual on earlier prompts predicts its residual on new ones. We therefore learn the corrections and the measured set from calibration data $\mathcal D=\{(\Delta_{pi},g_{pi})\}$, the exact AP measurements and \atp{} estimates of component $i$ on prompt $p$ available for fitting, from earlier prompts or from the test prompt itself, with residuals $e_{pi}=\Delta_{pi}-g_{pi}$. We fit the residual, and then the size of what the fitted correction leaves, by absolute loss over $\mathcal D$,
\begin{equation}
\begin{aligned}
\widehat f&=\arg\min_{f}\|e-f\|_{1},\\
\widehat h&=\arg\min_{h}\big\||e-\widehat f|-h\big\|_{1},
\end{aligned}
\label{eq:fit}
\end{equation}
where $f$ and $h$ are evaluated at $(g_{pi},i)$ and may belong to any model class, such as a constant, a linear regression, or a neural network; in this paper $f$ is a constant and $h$ a log-linear model (Appendix~\ref{app:acquisition}). At test time, for a test prompt $p'$, each component $i$ receives a correction and a selection score,
\begin{equation}
\widehat c_{p'i}=\widehat f(g_{p'i},i),\qquad s_{p'i}=\widehat h(g_{p'i},i).
\label{eq:score}
\end{equation}
The budget goes to the components with the highest scores, and the rest receive corrections $g_{p'i}+\widehat c_{p'i}$. We call this procedure \guided{}.

The performance of \guided{} depends on what can be learned from measured residuals, since a measured residual helps unmeasured components only insofar as it reveals an error pattern they share. The patterns in Figure~\ref{fig:motivation} suggest two axes of sharing: across components within one prompt, and across prompts of the same task for a given component. These lead to two learning settings, which we call \textit{within-prompt learning} and \textit{across-prompt learning}, respectively. The former learns a prompt-specific offset from other components measured on that prompt; the latter learns component-specific error patterns from earlier prompts, informing both correction and verification on later prompts.
Within-prompt learning needs no prior labels, whereas across-prompt learning requires measurements from earlier prompts. We describe each setting's labels, correction $\widehat c$, and selection score $s$ below; Algorithm~\ref{alg:method} summarizes both.

\paragraph{Within-prompt learning.}
\label{sec:single-prompt}
This setting learns a prompt-specific offset from a few components measured on the same prompt, so it needs no prior labels but pays for its own. The budget is split into a pilot set $S_1$ that supplies the labels and a verification set $S_2$, with $|S_1|+|S_2|=B$. The pilot should represent the typical component, since the offset it estimates will be applied to all unmeasured ones. Because residual size varies with attribution, we stratify by $|g_{p'i}|$, dividing the components into $K$ equal groups $G_1,\ldots,G_K$ by attribution rank and drawing the pilot equally from each,
\begin{equation}
S_1=\bigcup_{k=1}^{K}U_k,\qquad U_k\subseteq G_k,
\label{eq:pilot}
\end{equation}
giving labels $\mathcal D=\{(\Delta_{p'j},g_{p'j})\}_{j\in S_1}$. The correction is this offset, estimated by the pilot median, which minimizes the absolute error \citep{koenker1978quantiles} and is insensitive to its few large residuals,
\begin{equation}
\widehat c_{p'i}=\operatorname{median}_{j\in S_1}e_{p'j}\quad\text{for all }i.
\label{eq:pilot-median}
\end{equation}
With at most one label per component, no component-specific size model can be fit, so the selection score is $s_{p'i}=|g_{p'i}|$, the only available proxy for residual size, and the remaining budget verifies the highest scores outside $S_1$. The pilot corrects the bulk; verification handles the tail.

\begin{algorithm}[t]
\caption{\guided{}}
\label{alg:method}
\small
\begin{algorithmic}[1]
\Require test prompt $p'$, budget $B$, calibration data $\mathcal D$ (possibly empty)
\State $g_{p'i}\gets\partial_iF_{p'}(0)$ for all $i$
\If{$\mathcal D=\varnothing$} \Comment{within-prompt learning}
  \State $S_1\gets$ stratified sample by $|g_{p'i}|$ \Comment{Eq.~\ref{eq:pilot}}
  \State $\mathcal D\gets\{(\Delta_{p'i},g_{p'i})\}_{i\in S_1}$
  \State $\widehat f\gets$ pilot median, $\widehat h\gets|g_{p'i}|$ \Comment{Eq.~\ref{eq:pilot-median}}
\Else \Comment{across-prompt learning}
  \State $S_1\gets\varnothing$
  \State $\widehat f,\widehat h\gets$ fit on $\mathcal D$ \Comment{Eq.~\ref{eq:fit}}
\EndIf
\State $\widehat c_{p'i}\gets\widehat f(g_{p'i},i)$, \ $s_{p'i}\gets\widehat h(g_{p'i},i)$ for all $i$ \Comment{Eq.~\ref{eq:score}}
\State $S_2\gets$ top $B-|S_1|$ of $s_{p'i}$ outside $S_1$
\State $S\gets S_1\cup S_2$
\State \Return $\widehat\Delta_{p'i}$ with estimated $\widehat c_{p'i}$ \Comment{Eq.~\ref{eq:estimator}}
\end{algorithmic}
\end{algorithm}

\paragraph{Across-prompt learning.}
\label{sec:many-prompt}
This setting learns from earlier prompts of the same task, so it needs labels but spends the whole test-time budget on verification. When a component's residual recurs across prompts, as in Figure~\ref{fig:motivation}, its residual on a new prompt can be predicted from earlier ones.
From these labels, $\widehat f$ in Equation~\ref{eq:fit} learns a per-component correction, and $\widehat h$ learns, from the component's attribution and identity, the size of the residual that correction leaves. Every component receives its correction $\widehat c_{p'i}$ and selection score $s_{p'i}$, and the $B$ components with the highest scores are measured exactly. The correction helps only when its labels are representative, meaning that each corrected component has enough labels measured uniformly across prompts. Otherwise the correction can mislead and raise the error, and we set $\widehat f=0$ in this case. For a component's average effect over many test prompts, we can correct the averaged \atp{} estimate by adding the component's mean labeled residual, with no exact measurement on these prompts. Further details are in Appendix~\ref{app:acquisition}.

\subsection{Theoretical Analysis}
\label{sec:theory}
We ask when each setting beats verification alone and why selection needs its own model. All results assume exact AP measurements; proofs are in Appendix~\ref{app:math}. Within a prompt, the pilot buys a shared correction at the price of displaced verification. Let $U$ be the unmeasured components, $A_U(c)=\sum_{i\in U}|e_i-c|$ their error after a shift $c$, and $m_U$ their median, the best shift; let $H_U=A_U(0)-A_U(m_U)$ be the gain of that shift, $E_U=A_U(\widehat c)-A_U(m_U)$ the estimation cost of the pilot median $\widehat c$, and $C$ the error removed by displaced verification.

\begin{theorem}[Within-prompt gain]
\label{thm:within}
Let $\widehat\Delta$ be the within-prompt estimate with a stratified pilot of $n$ labels, $n/K$ from each of $K$ equal-size strata, and $\widehat\Delta^0$ the estimate that spends the same budget verifying the largest $|g_i|$ with no correction (verify-more, Section~\ref{sec:setup}). Then
\begin{equation}
N\big(\operatorname{MAE}(\widehat\Delta^0)-\operatorname{MAE}(\widehat\Delta)\big)=H_U-E_U-C,
\label{eq:pilot-tradeoff}
\end{equation}
and, over the pilot draw, for all $t\in(0,\tfrac12)$, with probability at least $1-2e^{-2nt^2}$, $E_U\le\varepsilon_U(t)=\max\{A_U(q(\tfrac12-t)),A_U(q(\tfrac12+t))\}-A_U(m_U)$, where $q$ is the quantile function of all $N$ residuals.
\end{theorem}

The pilot pays when the unmeasured residuals share a sign, so $H_U$ is large, and cluster near their median, so $E_U$ is small, by more than $C$.  We first analyze a component-specific median correction learned from representative calibration labels. This result concerns absolute error on individual prompts. Correcting dataset averages instead requires estimating the mean residual. For independent prompt residuals with quantile function $q$ and median $m$, let $R(c)=\E|e-c|$ and $\Gamma=R(0)-R(m)$ be the best shift's gain.

\begin{theorem}[Median correction from representative labels]
\label{thm:generalization}
 Let $\widehat c$ be the median of $n$ labels drawn independently from the component's residual law, and assume $\E|e|<\infty$. For any $t\in(0,\tfrac12)$, with probability at least $1-2e^{-2nt^2}$,
\begin{equation}
R(\widehat c)\le R(0)-\Gamma+2t\,w(t),
\label{eq:generalization}
\end{equation}
where $w(t)=q(\tfrac12+t)-q(\tfrac12-t)$.
\end{theorem}

 The bound explains when recurring residual bias makes a median correction useful. Magnitude-guided acquisition need not satisfy its sampling assumption. For dataset-average correction, an offset improves a component's estimate when it has the same sign as the average residual and less than twice its magnitude (Appendix~\ref{app:selection-bias}). Selection asks a different question. Equation~\ref{eq:utility} verifies the components whose error after correction is largest, so a score must predict that error's size. Fix the corrections and let $r_i=\E|e_i-\widehat c_i|$ be component $i$'s expected remaining error given what is known before measurement.

\begin{proposition}[Selection needs size]
\label{prop:selection}
If $\max_i|s_i-r_i|\le\epsilon$, selecting the $B$ highest scores costs at most $2B\epsilon$ extra expected error over the $B$ highest $r_i$.
\end{proposition}

Scores must track the expected size of the remaining error, not the size of the expected error. A component whose residual is large but changes sign has $\widehat f$ near zero and large $r_i$. This is why $h$ is fit separately from $f$ in Equation~\ref{eq:fit}, and why a single prompt, which cannot fit $h$, falls back to $|g|$.

\section{Experiments}
\label{sec:results}
\subsection{Experimental Setup}
\label{sec:setup}

\begin{figure*}[t]
\centering
\includegraphics[width=\textwidth]{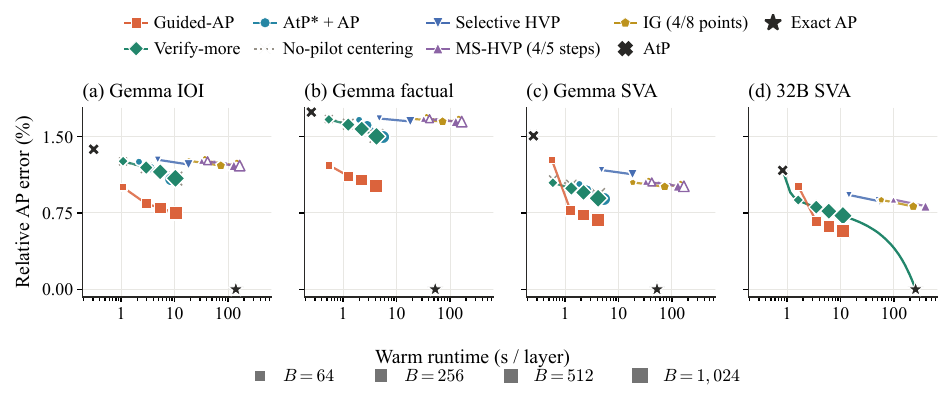}
\caption{Warm runtime and relative AP error for within-prompt \guided{} and the baselines. Square size denotes the measurement budget $B$; hollow derivative markers denote larger MS-HVP and IG settings. Appendix~\ref{app:single-prompt-comparators} gives the timing details.}
\label{fig:baseline-comparison}
\end{figure*}
\paragraph{Models, tasks, and prompts.}
For within-prompt \guided{} (Section~\ref{sec:res:single}), we study MLP neurons of Gemma-2-9B \citep{gemmateam2024gemma2}, Qwen3-14B, and Qwen3-32B \citep{yang2025qwen3}. For across-prompt \guided{} (Section~\ref{sec:res:across}), we study attention heads of Qwen3-8B and Qwen2.5-7B-Instruct \citep{qwen2025qwen25}. Tasks: indirect-object identification (IOI), factual completion, and subject-verb agreement (SVA). Score: logit difference between correct and counterfactual tokens. Evaluation prompts are held out; across-prompt calibration uses separate prompts in one prespecified evaluation. Figure~\ref{fig:motivation} uses separate exact effects from GPT-2.

\paragraph{Methods and baselines.}
Within-prompt \guided{} uses 256 exact measurements per layer. Across-prompt \guided{} verifies $B\in\{16,32,64,128\}$ heads per prompt and separately corrects dataset averages without measuring the evaluation prompts. Our main baseline, \textit{verify-more}, spends the same cost on exact measurements of the components with the largest \atp{} magnitudes and leaves the rest uncorrected. Across-prompt labels are acquired on large attributions, so verification sets $\widehat f=0$ (Section~\ref{sec:many-prompt}) and its comparison with verify-more isolates selection. An alternative calibration signal, \probecal{}, replaces the exact calibration labels with curvature-probe sketches, randomized estimates of the Hessian diagonal \citep{bekas2007diagonal}, and fits the size model and average correction from them (Appendix~\ref{app:probecal}). On Gemma-2-9B, we also compare with verify-more on \atpstar{} estimates, selective HVP, MS-HVP, singleton IG, and exhaustive AP (Appendices~\ref{app:single-prompt-comparators} and~\ref{app:protocol}). All costs include pilot, calibration, and fitting. We report measured warm runtime per layer within prompts, and calibration plus projected deployment cost across prompts.

\paragraph{Metrics and uncertainty.}
Primary: per-prompt effect-vector MAE against exact FP32 AP (Section~\ref{sec:preliminaries}). Additional metrics: relative AP error (summed absolute error over summed absolute effects), dataset-average error, and recall of the ten largest effects. Error reduction: one minus the ratio of mean errors. Intervals: paired group bootstraps, adjusted over the prespecified across-prompt comparison family and descriptive within prompts (Appendices~\ref{app:single-prompt-evaluation}, \ref{app:protocol}, and~\ref{app:cost}).

\setcounter{topnumber}{2}
\setcounter{dbltopnumber}{1}
\renewcommand{\topfraction}{0.45}
\subsection{Within-Prompt Learning}
\label{sec:res:single}
Within-prompt \guided{} lowers Gemma-2-9B MAE by 17\% to 26\% across all three tasks at the same measured runtime as verify-more, without earlier calibration prompts (Table~\ref{tab:single-prompt}). Top-ten recall remains nearly perfect for both methods.

\begin{table}[H]
\centering

\small
\setlength{\tabcolsep}{2.5pt}
\begin{tabular}{lrrrr}
\toprule
Task & Groups & \shortstack{Lower\\MAE} & \shortstack{Relative\\AP error} & \shortstack{Warm time\\(s/layer)}\\
\midrule
\multicolumn{5}{l}{Gemma-2-9B}\\
IOI & 100 & 17.4\% & 0.94\% & 2.08\\
Factual & 100 & 26.2\% & 1.20\% & 1.30\\
SVA & 100 & 19.9\% & 0.71\% & 1.25\\
\midrule
\multicolumn{5}{l}{Qwen3-14B}\\
IOI & 52 & 24.0\% & 1.34\% & 2.95\\
Factual & 56 & 12.3\% & 0.57\% & 1.67\\
SVA & 53 & 10.9\% & 0.58\% & 1.64\\
\midrule
\multicolumn{5}{l}{Qwen3-32B}\\
IOI & 10 & 17.3\% & 1.09\% & 6.52\\
Factual & 31 & 17.5\% & 1.09\% & 3.54\\
SVA & 40 & 13.7\% & 0.66\% & 3.55\\
\bottomrule
\end{tabular}
\caption{Within-prompt \guided{} accuracy and warm runtime. MAE reductions compare with runtime-matched verify-more. Groups is the number of evaluation prompts scored on all three tested MLP layers (Appendix~\ref{app:single-prompt-evaluation}).}\label{tab:single-prompt}\label{tab:larger-models}
\end{table}

\paragraph{Correction mechanism.}
Correcting unmeasured estimates drives the gain. The fitted shift captures 96\% to 97\% of the best common shift's benefit, outweighing displaced verification (Theorem~\ref{thm:within}; Appendix~\ref{app:mechanism-precision}). Gains persist at equal measurement counts (Appendix~\ref{app:single-prompt-evaluation}). FP64 readout rescoring of the same estimates retains 10\% to 23\% lower Gemma MAE and positive intervals on every larger-model task (Appendix~\ref{app:mechanism-precision}).

\begin{figure}[t]
\centering
\includegraphics[width=\columnwidth]{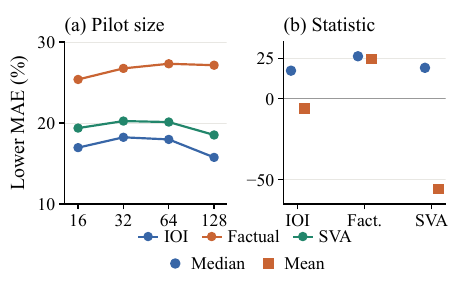}
\caption{Within-prompt \guided{} ablations on Gemma-2-9B at 256 measurements: pilot size (left) and mean versus median (right). Values show MAE reductions versus verification alone. Cohorts: Appendix~\ref{app:cpu-controls}.}
\label{fig:pilot-ablations}
\end{figure}

\paragraph{Controls and ablations.}
Measurements chosen for verification need not represent unmeasured errors. Shifting the remaining estimates by the median residual of those measurements raises MAE in eight of nine settings. The pilot lowers Gemma MAE by 17\% to 24\% relative to the better of this control and verification alone (Table~\ref{tab:pilot-controls}).

\begin{table}[H]
\centering

\small
\setlength{\tabcolsep}{3pt}
\begingroup
\begin{tabular*}{\columnwidth}{@{\extracolsep{\fill}}lrrr@{}}
\toprule
Task & Verify only & Verify + shift & \guided{}\\
\midrule
IOI & 1.00 & 1.93 & \textbf{0.83}\\
Factual & 1.00 & 0.97 & \textbf{0.74}\\
SVA & 1.00 & 1.80 & \textbf{0.81}\\
\bottomrule
\end{tabular*}
\endgroup
\caption{Within-prompt \guided{} controls on Gemma-2-9B at the same 256-measurement budget. MAE is relative to verification alone; lower is better.}\label{tab:pilot-controls}
\end{table}

On Gemma, \guided{} lowers group-averaged MAE against runtime-matched verify-more in 96 of 100 IOI groups, 97 of 100 factual groups, and 98 of 100 SVA groups (Appendix~\ref{app:mechanism-precision}).

A modest pilot captures most gains. Pilots of 32 or 64 improve on 16, and 128 adds no consistent gain (Figure~\ref{fig:pilot-ablations}a; Appendix Table~\ref{tab:pilot-size-acquired}). Every tested pilot size lowers MAE on all three Gemma tasks, leaving flexibility in how the budget is divided between pilot measurements and exact verification. The median beats the mean on identical labels in all nine settings; uniform and magnitude-stratified samples are similar on Gemma (Figure~\ref{fig:pilot-ablations}b; Appendix~\ref{app:cpu-controls}).

\subsection{Across-Prompt Learning}
\label{sec:res:across}

We compare across-prompt \guided{} with \probecal{} (Section~\ref{sec:setup}). We assess verification against verify-more at matched total cost, and dataset-average correction against uncorrected \atp{}, without evaluation measurements for the latter.

\paragraph{Verification allocation.}
\label{sec:res:allocation}
Allocation needs transfer and reuse. On Qwen3 IOI, \guided{} lowers MAE by 12\% against verify-more at matched projected total cost, calibration included (Table~\ref{tab:allocation}).

\begin{table}[H]
\centering

\small
\setlength{\tabcolsep}{2pt}
\begin{tabular*}{\columnwidth}{@{\extracolsep{\fill}}llrrrr@{}}
\toprule
Task & Method & \multicolumn{2}{c}{Qwen3-8B} & \multicolumn{2}{c}{Qwen2.5-7B}\\
\cmidrule(lr){3-4}\cmidrule(lr){5-6}
 & & $B=16$ & $B=128$ & $B=16$ & $B=128$\\
\midrule
IOI & \guided{} & \textbf{12\%} & \textbf{7\%} & \textit{$-$29\%} & 12\%\\
 & \probecal{} & \textbf{17\%} & \textbf{19\%} & \textit{$-$39\%} & 15\%\\
\addlinespace[1pt]
Factual & \guided{} & \textit{$-$20\%} & 2\% & \textit{$-$18\%} & $-$2\%\\
 & \probecal{} & \textit{$-$17\%} & \textbf{31\%} & $-$5\% & \textbf{35\%}\\
\addlinespace[1pt]
SVA & \guided{} & \textbf{3\%} & \textbf{4\%} & \textbf{2\%} & \textbf{3\%}\\
 & \probecal{} & 2\% & \textbf{12\%} & 1\% & \textbf{19\%}\\
\bottomrule
\end{tabular*}
\caption{Across-prompt MAE reductions relative to verify-more at matched projected total cost, calibration included, for 10,000 prompts. Bold and italic mark adjusted significant improvements and harms.}\label{tab:allocation}
\end{table}

\label{sec:res:savings}
Where learning pays, projected break-even at the smallest budget ranges from 1,900 to 24,000 prompts across the two calibration signals. Shorter analyses favor additional verification; verify-more also remains strong on factual completion at small budgets and on Qwen2.5 IOI (Appendices~\ref{app:frontiers} and~\ref{app:cost}).

\paragraph{Dataset-average correction.}
\label{sec:res:dataset}
Persistent average errors can be corrected without measuring each new prompt. On factual completion, \guided{} lowers dataset-average MAE by 75\% to 79\% on both models relative to uncorrected \atp{} (Figure~\ref{fig:dataset-mean-main}; Appendix Table~\ref{tab:dataset}). Post hoc cost comparisons give projected payback of 500 to 2,700 prompts where correction helps; shorter analyses favor verify-more (Appendix~\ref{app:controls}).

Learning residuals makes better use of calibration measurements. Directly averaging the same AP effects produces 5.8 to 14.9 times more error than residual correction in every setting (Appendix~\ref{app:controls}). On factual completion, \guided{} labels cover 13\% to 18\% of heads; the rest retain their \atp{} averages.

Signed error transfer can fail. \guided{} raises Qwen2.5 IOI dataset-average MAE by 52\%, consistent with selection bias (Appendix~\ref{app:selection-bias}). SVA has no established correction gain, and neither correction establishes a dataset-average recall gain.

\paragraph{Calibration signals.}
\probecal{} helps allocation where exact-label learning is less effective, lowering factual MAE by 31\% to 35\% against verify-more at matched projected total cost (Table~\ref{tab:allocation}). Its factual average correction reduces MAE by 46\% against uncorrected \atp{}, less than \guided{} (Figure~\ref{fig:dataset-mean-main}); it helps Qwen2.5 IOI averages where \guided{} harms them (Appendix Table~\ref{tab:dataset}).

\paragraph{Error structure and recall.}
\label{sec:res:two}
Verification and average correction require different error information. Factual errors recur (Appendix~\ref{app:stability}), yet learned selection can lose at small budgets because magnitude ranking already captures the largest residuals. SVA improves verification without an established mean-correction gain. Lower MAE also need not improve top-ten recall (Appendix~\ref{app:frontiers}).

\begin{figure}[t]
\centering
\includegraphics[width=\columnwidth]{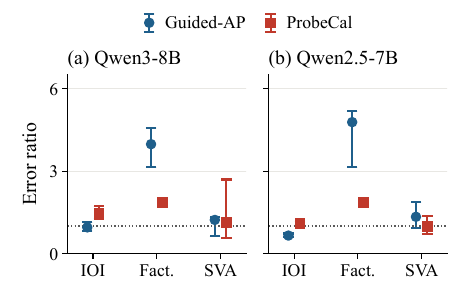}
\caption{Dataset-average correction on the two models. Ratios are uncorrected over corrected MAE, with adjusted intervals; the dotted line denotes equality.}
\label{fig:dataset-mean-main}
\end{figure}

\subsection{Scalability and Efficiency}
\label{sec:res:scalability}
We separate accuracy gains at a fixed measurement budget from runtime savings at an error target.

\paragraph{Component count and accuracy.}
With the measurement budget fixed, the benefit persists at every tested component count across all three Gemma tasks, up to 14,336 neurons (Appendix~\ref{app:component-scaling}).

Runtime savings depend on the required accuracy. At a 1\% mean relative AP error target, \guided{} is 7.9 times faster than verify-more on Gemma IOI in measured warm runtime. This runtime includes attribution, pilot measurements, and fitting on the same prompt; the gain requires no reuse across prompts. Verification is cheaper on SVA, and \guided{} misses the factual target (Table~\ref{tab:single-prompt-tolerance}; Appendix~\ref{app:single-prompt-tolerance}).

\begin{table}[H]
\centering

\small
\setlength{\tabcolsep}{3pt}
\begin{tabular}{lrrrr}
\toprule
& \multicolumn{2}{c}{\guided{}} & \multicolumn{2}{c}{Verification to 1\%}\\
\cmidrule(lr){2-3}\cmidrule(lr){4-5}
Task & Error & s/layer & Checks & s/layer\\
\midrule
IOI & 0.94\% & 2.08 & 2,416 & 16.37\\
Factual & \textit{1.20\%} & 1.30 & 5,664 & 21.84\\
SVA & 0.71\% & 1.25 & 32 & 0.43\\
\bottomrule
\end{tabular}
\caption{Warm cost at a 1\% mean relative AP error target on Gemma-2-9B. Italic errors exceed the target; checks count magnitude-ranked AP measurements.}\label{tab:single-prompt-tolerance}
\end{table}

\paragraph{Model size.}
The same budget yields 11\% to 24\% lower MAE on Qwen3-14B and 14\% to 18\% on Qwen3-32B at matched runtime, measuring only 1\% of the latter's neurons (Table~\ref{tab:larger-models}; Figure~\ref{fig:scaling-comparators}a). At \guided{}'s achieved error, a post hoc comparison finds verify-more takes 5.9 to 14.7 times longer in warm runtime across the three 32B tasks (Figure~\ref{fig:scaling-comparators}b; Appendix Table~\ref{tab:32b-matched-accuracy}). This frontier comparison uses the same prompts and matches the correction's observed error, separately from the prespecified targets.

\paragraph{Comparison with derivative corrections.}
On the 30-group Gemma comparison set, \guided{} has lower MAE than selective HVP, MS-HVP, and singleton IG at their tested settings, taking about one to three warm seconds per layer (Figure~\ref{fig:baseline-comparison}; Appendix~\ref{app:single-prompt-comparators}). \atp{} remains fastest; exhaustive AP removes approximation error at greater cost.

\begin{figure}[t]
\centering
\includegraphics[width=\columnwidth]{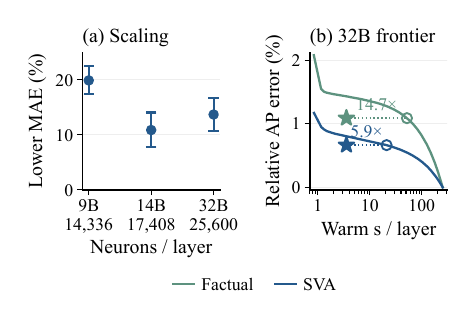}
\caption{Scaling of the fixed budget. Left, SVA MAE reductions at matched cost with descriptive intervals. Right, measured 32B verification curves for factual completion and SVA (lines), within-prompt \guided{} (stars), and verification at the same error (open circles), on the same complete groups; IOI: Appendix Table~\ref{tab:32b-matched-accuracy}.}
\label{fig:scaling-comparators}
\end{figure}

\section{Conclusion and Limitations}
Attribution patching trades the accuracy of activation patching for speed. \guided{} uses a few exact measurements to improve estimates beyond the ones it verifies. On a single prompt, the median error of a small exact sample corrects the unmeasured estimates, as long as enough budget remains to verify the largest ones. Over many prompts of one task, a few measured prompts reveal each component's recurring error. Its size guides verification on later prompts, and its direction corrects average effects, which repays the cost of learning in long analyses. Future work could combine the two settings and test them on sparse features, where activation patching is most costly.

Our evaluation focuses on transformer language models and controlled language tasks. Further evaluation on other architectures, component representations, and open-ended language behaviors would clarify its broader applicability. Accuracy and computational costs depend on the task, measurement budget, hardware, and implementation. Activation patching measures effects relative to a chosen intervention and output metric. These estimates describe individual components and can inform mechanistic explanations alongside tests of how components work together.

\bibliography{references}
\appendix
\setcounter{topnumber}{2}
\renewcommand{\topfraction}{0.7}
\section{Label Acquisition and Fits}
\label{app:acquisition}
This appendix specifies label acquisition and the fitted correction $\widehat f$ and size model $\widehat h$. Acquired AP values and residuals are written $\widetilde\Delta_i$ and $\widetilde e_i=\widetilde\Delta_i-g_i$. They equal the scoring reference $\Delta_i$ and residual $e_i$ in the primary analysis; Appendix~\ref{app:mechanism-precision} distinguishes them for FP64 rescoring.

\subsection{Within-Prompt Learning}
\label{app:single-prompt}
\textbf{Pilot.} Sort components by increasing $|g_i|$, breaking ties by component order, and partition them into $K=8$ equal-sized rank strata. Independently permute each stratum and take one component per stratum in round-robin order until the pilot is full (Equation~\ref{eq:pilot}). All evaluated populations are divisible by eight.

\textbf{Correction and selection.} Fit one offset, the median acquired pilot residual (Equation~\ref{eq:pilot-median}). Set $\widehat h=|g_{p'i}|$ and use the remaining $B-|S_1|$ measurements on the largest magnitudes outside the pilot. Runtime includes both measurement sets, selection, fitting, and estimate assembly.

\subsection{Across-Prompt Learning}
\label{app:many-prompt}
\textbf{Label acquisition.} On each calibration prompt, sort components by decreasing $|g|$. Visit prompts in fixed round-robin order, taking the next unacquired component until the label budget is exhausted.

\textbf{Size model $\widehat h$.} Each acquired label supplies $x=\log(|g|+\eta)$ and $y=\log(|\widetilde e-\widehat c_i|+\eta)$ with $\eta=10^{-12}$, where $\widehat c_i$ is the correction fit below; with $\widehat f=0$, as in our experiments, $y=\log(|\widetilde e|+\eta)$. A pooled least-squares fit gives slope $a$ and intercept $\beta$; component $i$ then receives its own intercept $\beta_i=\beta+w_i(\operatorname{mean}_i[y-ax]-\beta)$ with $w_i=n_i/(n_i+5)$, where $n_i$ is its number of labels, and the pooled intercept when it has none. The score is
\begin{equation*}
\widehat h(g,i)=\exp\big(a\log(|g|+\eta)+\beta_i\big),
\end{equation*}
a log-linear model fit by least squares in log space as a surrogate for the absolute-loss objective of Equation~\ref{eq:fit}. The score approximates expected remaining error $r_i$, the target of Proposition~\ref{prop:selection}.

\textbf{Correction $\widehat f$.} Equation~\ref{eq:fit} restricted to one constant per component is minimized by the median of each component's acquired residuals. Our verification experiments set $\widehat f=0$ (Section~\ref{sec:setup}).

\textbf{Dataset averages.} Let $\mathcal C_i$ be the calibration prompts on which component $i$ was measured. Its average correction is the mean acquired residual,
\begin{equation}
\bar c_i=
\begin{cases}
|\mathcal C_i|^{-1}\sum_{p\in\mathcal C_i}e_{pi},&|\mathcal C_i|>0,\\
0,&|\mathcal C_i|=0,
\end{cases}
\label{eq:dataset-mean-correction}
\end{equation}
with no shrinkage, and the corrected average is $\bar g_i+\bar c_i$. The target is the average residual over prompts, so it is estimated by an average; Lemma~2 states when such a correction helps. For probes, Equation~\ref{eq:probe-mean-correction} plays the same role.

\section{Curvature-Probe Calibration}
\label{app:probecal}
\probecal{} supplies an alternative learning signal through shared curvature probes. When $F_p$ is twice differentiable at the baseline, let $G_p=\nabla^2F_p(0)$ be its Hessian in intervention coordinates for calibration prompt $p$. With independent random sign vectors $z_{pr}\in\{-1,1\}^N$, its sketch is
\begin{equation*}
d_{pi}=\frac1R\sum_{r=1}^R z_{pr,i}(G_pz_{pr})_i.
\end{equation*}
The randomized diagonal estimator of \citet{bekas2007diagonal} is computed with Hessian-vector products \citep{pearlmutter1994hvp}. Its expectation is baseline curvature; Equation~\ref{eq:variance} gives its interaction-dependent noise, and Equation~\ref{eq:remainder} distinguishes curvature from the finite AP residual. For $J$ calibration prompts with $R$ probes each, the signed average correction is
\begin{equation}
\bar c_i^{\rm probe}=\frac1{2J}\sum_{p=1}^J d_{pi},
\label{eq:probe-mean-correction}
\end{equation}
and the corrected average is $\bar g_i+\bar c_i^{\rm probe}$. Averaging retains the signs; a sketch estimates baseline curvature, not the complete finite-AP residual.
\probecal{} uses the sketches as labels for the size model. Every component-prompt pair supplies $x_{pi}=\log(|g_{pi}|+\eta)$ and $y_{pi}=\log(|d_{pi}|+\eta)$, where $d_{pi}$ averages the signed probes before the magnitude is taken. A pooled least-squares fit gives $(a,\beta)$; a component with varying log attribution across prompts also receives its own pair $(a_i,\beta_i)$, shrunk toward the pooled values, and its score is $\widehat h(g,i)=\exp(a_i\log(|g|+\eta)+\beta_i)$. A capacity-matched AP-only control, fit the same way on exact labels, separates the label source from the model class. Probes are computed only on calibration prompts; no probe is run on evaluation prompts.

\paragraph{Probe noise and the finite-intervention remainder.}
For a fixed Hessian $G$ and independent sign vectors, $z_i(Gz)_i=G_{ii}+\sum_{j\ne i}G_{ij}z_iz_j$, so conditional on $G$ \citep{bekas2007diagonal}
\begin{equation}
\E[d_i\mid G]=G_{ii},\qquad \Var(d_i\mid G)=\frac1R\sum_{j\ne i}G_{ij}^2,
\label{eq:variance}
\end{equation}
and $\operatorname{Cov}(d_i,d_k\mid G)=G_{ik}^2/R$ for $i\ne k$. Interactions between components set the noise of each sketch and correlate sketches across components, so they are not independent samples. The sketch also targets only the curvature term of the residual. For a twice differentiable path $f_i(t)=F(tv_i)$, Taylor's theorem gives
\begin{equation}
\begin{aligned}
e_i&=\tfrac12G_{ii}+\rho_i,\\
\rho_i&=\int_0^1(1-t)\{f_i''(t)-f_i''(0)\}\,dt,
\end{aligned}
\label{eq:remainder}
\end{equation}
Thus the sketches supply a curvature-based learning signal rather than exact finite-intervention residual labels.

\section{Mathematical Details}
\label{app:math}

\subsection{Proof of Theorem~\ref{thm:within}}
\label{app:pilot-proof}
We extend the identity to the runtime-matched baseline of the experiments, which may measure $B_0\ne B$ components, and acquired AP values $\widetilde\Delta_i$ that may differ from the reference $\Delta_i$. Fix $g$ and $\Delta$ and let $e=\Delta-g$. The method measures $S$ with $|S|=B$ and shifts the unmeasured set $U$ by $\widehat c$; the baseline measures $V_{B_0}$, the $B_0$ largest $|g_i|$, and leaves the rest at $g_i$. Let $D_{\rm acq}=\sum_{i\in S}|\widetilde\Delta_i-\Delta_i|-\sum_{i\in V_{B_0}}|\widetilde\Delta^0_i-\Delta_i|$ be the difference in acquisition error, and $C=\sum_{i\in V_{B_0}}|e_i|-\sum_{i\in S}|e_i|$ the verification opportunity cost. For the decomposition reported in Appendix~\ref{app:single-prompt-evaluation}, $C$ splits into an allocation term and a count term,
\begin{equation*}
\begin{aligned}
C_{\rm alloc}&=\sum_{i\in V_B}|e_i|-\sum_{i\in S}|e_i|,\\
C_{\rm count}&=\sum_{i\in V_{B_0}}|e_i|-\sum_{i\in V_B}|e_i|,
\end{aligned}
\end{equation*}
which isolate the change in which $B$ components are measured and the extra or fewer measurements affordable at matched runtime.

\textbf{The identity.} The summed errors of the two estimates are
\begin{equation*}
\begin{aligned}
L_0&=\sum_{i\notin V_{B_0}}|e_i|+\sum_{i\in V_{B_0}}|\widetilde\Delta^0_i-\Delta_i|,\\
L_1&=A_U(\widehat c)+\sum_{i\in S}|\widetilde\Delta_i-\Delta_i|.
\end{aligned}
\end{equation*}
Subtracting, $L_0-L_1=A_U(0)-A_U(\widehat c)-C-D_{\rm acq}$, and adding and subtracting $A_U(m_U)$ gives
\begin{equation*}
N(\operatorname{MAE}_0-\operatorname{MAE}_1)=H_U-E_U-C-D_{\rm acq}.
\end{equation*}
Since $A_U$ is convex with a minimum at any median, $H_U,E_U\ge0$. When acquired values equal the reference, $D_{\rm acq}=0$ and Equation~\ref{eq:pilot-tradeoff} follows. $\square$

\textbf{The estimation cost.} The bound on $E_U$ uses one fact about convex functions.

\noindent\textbf{Lemma 1.} If $\widehat c\in[a,b]$, then $A_U(\widehat c)\le\max\{A_U(a),A_U(b)\}$.

\textbf{Proof.} $A_U$ is convex, and every point of $[a,b]$ is a convex combination of its endpoints. $\square$

Let $q$ be the quantile function of the $N$ residuals and $n$ the pilot size. The pilot is drawn without replacement, equally from $K$ equal-sized strata (Equation~\ref{eq:pilot}), with independent draws across strata, so the number of pilot residuals below $q(\tfrac12-t)$ is a sum over strata of independent without-replacement counts with mean at most $n(\tfrac12-t)$. Hoeffding's inequality holds for sampling without replacement and for independent sums of such counts \citep{hoeffding1963}, so the pilot median falls below $q(\tfrac12-t)$ with probability at most $e^{-2nt^2}$, and symmetrically above $q(\tfrac12+t)$. On the complementary event, Lemma~1 with $[a,b]=[q(\tfrac12-t),q(\tfrac12+t)]$ gives
\begin{equation*}
E_U\le\max\{A_U(a),A_U(b)\}-A_U(m_U)\eqqcolon\varepsilon_U(t),
\end{equation*}
the excess error of the worst shift in the band. The probability statement concerns only the pilot median against the quantiles of all $N$ residuals; Lemma~1 then applies to whichever unmeasured set $U$ the pilot and verification leave. $\square$

\subsection{Proof of Theorem~\ref{thm:generalization}}
\label{app:generalization-proof}
 The theorem assumes $n$ independent labels from the component's residual law and $\E|e|<\infty$. Labels acquired by attribution magnitude need not satisfy this, and the band is undefined for $t\ge\tfrac12$. Let $e_1,\ldots,e_n$ be the labels and $\widehat c$ their median.

\emph{Step 1: sample median in the central band.} If $\widehat c<q(\tfrac12-t)$, at least $n/2$ labels fall below $q(\tfrac12-t)$. Their number is binomial with success probability at most $\tfrac12-t$, so Hoeffding's inequality bounds this event by $e^{-2nt^2}$; similarly for $\widehat c>q(\tfrac12+t)$.

\emph{Step 2: excess error inside the band.} $R(c)=\E|e-c|$ is convex with derivative $2\Phi(c)-1$ almost everywhere, where $\Phi$ is the distribution function of $e$, so for $\widehat c\ge m$,
\begin{equation*}
\begin{aligned}
R(\widehat c)-R(m)&=\int_m^{\widehat c}\bigl(2\Phi(u)-1\bigr)\,du\\
&\le2t\,(\widehat c-m)\le2t\,w(t),
\end{aligned}
\end{equation*}
since $\tfrac12\le\Phi(u)\le\tfrac12+t$ on the band. The case $\widehat c<m$ is symmetric.

\emph{Step 3.} Since $R(m)=R(0)-\Gamma$, the two steps give Equation~\ref{eq:generalization}. A union bound extends the event to all $N$ components with probability at least $1-2Ne^{-2nt^2}$. \hfill$\square$

\subsection{Proof of Proposition~\ref{prop:selection}}
\label{app:selection-proof}
Let $T$ be the $B$ components with the largest $r_i$ and $S$ the $B$ with the largest scores $s_i$. The expected error left by $S$ exceeds that left by $T$ by $\sum_{i\in T}r_i-\sum_{i\in S}r_i$. Since $r_i\le s_i+\epsilon$ on $T$ and $r_i\ge s_i-\epsilon$ on $S$,
\begin{equation*}
\sum_{i\in T}r_i-\sum_{i\in S}r_i\le\sum_{i\in T}s_i-\sum_{i\in S}s_i+2B\epsilon\le2B\epsilon,
\end{equation*}
because $S$ maximizes the score sum over sets of size $B$. $\square$

The oracle knows expected remaining errors before measurement. A residual that changes sign can have a small fitted correction but large expected absolute error, so a correction alone need not provide a useful selection score.

\subsection{Dataset Averages: Sign Condition and Selected Labels}
\label{app:selection-bias}

\noindent\textbf{Lemma 2 (Sign condition).} For an averaged residual $\bar e_i$ and correction $\bar c_i$, $|\bar e_i-\bar c_i|<|\bar e_i|$ if and only if $\bar c_i\bar e_i>0$ and $|\bar c_i|<2|\bar e_i|$.

\textbf{Proof.} Squaring, the inequality is equivalent to $\bar c_i^2<2\bar c_i\bar e_i$, which requires $\bar c_i\bar e_i>0$ and then reduces to $|\bar c_i|<2|\bar e_i|$. $\square$

Dataset averages require the mean residual, whereas per-prompt absolute loss is minimized by a median. Acquisition can bias the mean. If $A_i$ indicates measurement of component $i$, the acquired-label mean targets $\E[e_i\mid A_i=1]$ rather than $\E[e_i]$. Section~\ref{sec:res:dataset} evaluates whether the acquired residuals transfer to dataset averages.

\section{Within-Prompt Evaluation and Controls}
\label{app:single-prompt-evaluation}
\subsection{Cohort, Acquisition, and Aggregation}
Within-prompt evaluation comprises 1,662 layer evaluations and 542 complete model-task-prompt groups, with a fixed $P=64$ pilot inside $B=256$ measurements. We intervene on all MLP neurons in a layer. Gemma-2-9B has 14,336 neurons (layers 10, 20, 30), Qwen3-14B has 17,408 (9, 19, 29), and Qwen3-32B has 25,600 (15, 31, 47). Each task has 100 prompt groups. Summaries include groups with all three layers; Gemma is complete, and available Qwen counts appear in Table~\ref{tab:single-prompt}.

Relative AP error divides summed absolute estimation error by summed absolute reference effects for each vector. MAE, dataset-average error, and top-ten recall retain their definitions in Section~\ref{sec:preliminaries}.

We average three fixed pilot repetitions within each layer, then layers within each prompt group, then groups. Error reduction is one minus the ratio of mean MAEs. Descriptive 95\% intervals use 20,000 paired group-bootstrap draws (seed 20261010), holding fits and measured costs fixed; across-prompt multiplicity adjustment is separate. Gemma intervals are 14.7\% to 20.1\% (IOI), 23.5\% to 28.9\% (factual), and 17.4\% to 22.5\% (SVA).

The method and baseline acquire separate AP measurements in the same allocation. Warm runtime includes attribution, acquisition, selection, fitting, and assembly, excluding loading, numerical validation, and reference collection. The baseline uses the largest executed magnitude-ranked batch prefix affordable at the method's runtime; equal-count comparisons are separate. Gemma, Qwen3-14B, and Qwen3-32B use one, two, and four L40S GPUs, respectively, with timings matched within each hardware configuration.

\subsection{Measured Cost at Specified Accuracy}
\label{app:single-prompt-tolerance}
The prespecified targets are 0.5\%, 1\%, and 2\% mean relative AP error. On the main cohorts (Table~\ref{tab:single-prompt}), each measured magnitude-ranked AP prefix retains its error and warm runtime, including attribution setup. At the exhaustive endpoint, we use the cheaper separately measured full-AP time, which omits attribution.

For each target, select the cheapest measured prefix whose group-averaged error meets it, applying the same count to all included layers and groups. Compare with fixed $P=64$, $B=256$ \guided{} only when both meet the target. Aggregation follows Appendix~\ref{app:single-prompt-evaluation}; targets constrain the cohort mean rather than every prompt.

\begin{table*}[t]
\centering

\small
\setlength{\tabcolsep}{3pt}
\begin{tabular}{llrrrccc}
\toprule
& & & \multicolumn{2}{c}{\guided{}} & \multicolumn{3}{c}{Cheapest measured verification prefix: checks / s per layer}\\
\cmidrule(lr){4-5}\cmidrule(lr){6-8}
Model & Task & Groups & Error & s/layer & 2\% target & 1\% target & 0.5\% target\\
\midrule
Gemma-2-9B & IOI & 100 & 0.94\% & 2.08 & 0 / 0.27 & 2,416 / 16.37 & 8,128 / 54.16\\
Gemma-2-9B & Factual & 100 & 1.20\% & 1.30 & 0 / 0.25 & 5,664 / 21.84 & 9,904 / 38.00\\
Gemma-2-9B & SVA & 100 & 0.71\% & 1.25 & 0 / 0.25 & 32 / 0.43 & 5,472 / 20.25\\
\midrule
Qwen3-14B & IOI & 52 & 1.34\% & 2.95 & 16 / 0.76 & 8,016 / 73.12 & 12,608 / 114.56\\
Qwen3-14B & Factual & 56 & 0.57\% & 1.67 & 0 / 0.42 & 0 / 0.42 & 2,144 / 10.29\\
Qwen3-14B & SVA & 53 & 0.58\% & 1.64 & 0 / 0.39 & 16 / 0.54 & 1,936 / 9.30\\
Qwen3-32B & IOI & 10 & 1.09\% & 6.52 & 0 / 0.87 & 6,016 / 126.55 & 15,552 / 324.93\\
Qwen3-32B & Factual & 31 & 1.09\% & 3.54 & 16 / 1.15 & 6,736 / 67.22 & 15,920 / 157.58\\
Qwen3-32B & SVA & 40 & 0.66\% & 3.55 & 0 / 0.84 & 16 / 1.16 & 6,672 / 66.88\\
\bottomrule
\end{tabular}
\caption{Warm cost at all specified targets. Zero checks denotes ordinary \atp{}.}\label{tab:single-prompt-all-targets}
\end{table*}

Table~\ref{tab:single-prompt-all-targets} reports every target and setting; a cost advantage requires both methods to meet the same target.

Fixed-estimate FP64 rescoring retains Gemma IOI error below 1\% (0.95\%); verify-more needs 2,064 checks and 14.05 seconds, giving a 6.74 time ratio with unchanged execution and timings.

\paragraph{Matching the correction's achieved error.}
This post hoc comparison matches \guided{}'s achieved mean relative AP error rather than a prespecified target. For each 32B task and reference, select the cheapest measured verification prefix attaining that error on the same complete groups, using one prefix count throughout (Table~\ref{tab:32b-observed-frontier}).

\begin{table*}[t]
\centering

\small
\setlength{\tabcolsep}{4pt}
\begin{tabular}{lrrrrrrr}
\toprule
& & \multicolumn{4}{c}{Primary reference} & \multicolumn{2}{c}{FP64 reference}\\
\cmidrule(lr){3-6}\cmidrule(lr){7-8}
Task & Groups & Error & \shortstack{\guided{}\\(s/layer)} & \shortstack{Verify-more\\(s/layer)} & Ratio & Error & Ratio\\
\midrule
IOI & 10 & 1.093\% & 6.52 & 91.16 & 14.0 & 1.110\% & 16.3\\
Factual & 31 & 1.088\% & 3.54 & 51.95 & 14.7 & 1.099\% & 11.8\\
SVA & 40 & 0.664\% & 3.55 & 21.06 & 5.9 & 0.658\% & 2.6\\
\bottomrule
\end{tabular}
\caption{Post hoc 32B accuracy and warm runtime at \guided{}'s achieved error. Ratios divide verify-more time by \guided{} time. FP64 rescoring keeps the estimates and measured runtimes fixed.}\label{tab:32b-observed-frontier}\label{tab:32b-matched-accuracy}
\end{table*}

Each paired group-bootstrap draw recomputes the method's mean error and the cheapest matching prefix (20,000 draws, seed 20261010). Descriptive time-ratio intervals are 7.9 to 19.5, 10.9 to 18.2, and 4.2 to 8.1 for IOI, factual completion, and SVA; fixed-estimate FP64 intervals are 9.1 to 21.3, 5.4 to 17.8, and 1.5 to 4.7. Cohorts match Table~\ref{tab:single-prompt}.

Relative error normalizes each vector before group averaging, so its prompt weighting differs from unnormalized MAE. Matching achieved relative error is distinct from matching runtime or a prespecified accuracy target.

\subsection{Measured Comparisons with Other Correction Methods}
\label{app:single-prompt-comparators}
The comparison with other correction methods uses 30 complete prompt groups per Gemma task, each scored on all three layers, totaling 270 layer comparisons. These groups overlap the main 100-group sets (Table~\ref{tab:single-prompt-comparators}).

All selective methods use 256 components or AP measurements; \guided{} includes its 64-label pilot. \atpstar{} verification includes gradient-drop ranking. HVP and MS-HVP correct magnitude-selected components; MS-HVP uses four or five steps. Singleton IG uses four or eight Gauss-Legendre points on isolated component paths. Unselected components retain \atp{}. No-pilot centering shifts them by the negative median attribution. Runtimes include ranking, operators, and assembly on the same hardware; reference collection is separate.

\begin{table*}[t]
\centering

\small
\setlength{\tabcolsep}{4pt}
\begin{tabular}{lrrrrrr}
\toprule
& \multicolumn{2}{c}{IOI} & \multicolumn{2}{c}{Factual} & \multicolumn{2}{c}{SVA}\\
\cmidrule(lr){2-3}\cmidrule(lr){4-5}\cmidrule(lr){6-7}
Method & Error & Seconds & Error & Seconds & Error & Seconds\\
\midrule
AtP & 1.474\% & 0.29 & 1.828\% & 0.28 & 1.404\% & 0.29\\
Verify-more & 1.292\% & 2.61 & 1.725\% & 1.29 & 0.909\% & 1.67\\
\atp{}* + AP & 1.286\% & 3.46 & 1.721\% & 2.69 & 0.898\% & 2.62\\
No-pilot centering & 1.320\% & 2.61 & 1.722\% & 1.29 & 0.942\% & 1.67\\
Selective HVP & 1.321\% & 18.69 & 1.758\% & 18.70 & 1.029\% & 18.76\\
MS-HVP (4) & 1.314\% & 131.52 & 1.757\% & 131.57 & 0.930\% & 132.01\\
MS-HVP (5) & 1.313\% & 169.23 & 1.757\% & 169.15 & 0.928\% & 169.68\\
Singleton IG (4) & 1.313\% & 73.76 & 1.757\% & 73.31 & 0.924\% & 73.33\\
Singleton IG (8) & 1.313\% & 147.10 & 1.757\% & 146.56 & 0.924\% & 146.10\\
\guided{} & 0.952\% & 2.69 & 1.230\% & 1.30 & 0.723\% & 1.64\\
Exhaustive AP & 0.000\% & 125.24 & 0.000\% & 53.43 & 0.000\% & 69.23\\
\bottomrule
\end{tabular}
\caption{Measured comparator accuracy and warm seconds per layer on the complete Gemma comparator groups. Error is relative AP error. Parentheses identify MS-HVP steps or IG quadrature points.}\label{tab:single-prompt-comparators}
\end{table*}

Against runtime-matched verify-more, \guided{} lowers MAE by 16.0\%, 26.3\%, and 19.7\% on IOI, factual completion, and SVA in this comparator cohort. Table~\ref{tab:single-prompt-comparators} gives all methods' measured errors and runtimes.

Figure~\ref{fig:baseline-comparison} uses fixed budget-frontier cohorts: nine IOI, eight factual, and nine SVA Gemma groups, plus 45 SVA groups for 32B. Gemma $B=512$ points reuse the first 512 measurements from $B=1024$ sequences with the same pilot; time combines the measured acquisition time of those prefixes with the $B=256$ setup, fit, and assembly overhead. Derivative budgets are 64 and 256. Hollow Gemma markers denote five-step MS-HVP and eight-point IG; 32B uses endpoint HVP, four-step MS-HVP, and four-point IG.

\subsection{Within-Prompt Controls and Ablations}
\label{app:cpu-controls}
Verification-plus-shift fits a median residual from the same 256 magnitude-ranked AP measurements as verification alone, then shifts unmeasured estimates. \guided{} uses a 64-label representative pilot inside that budget. Mean-versus-median fits share labels. Uniform and stratified sampling are replayed from measured AP tables, counting all selected measurements. These controls use the main cohorts; the acquired pilot-size study uses a separate cohort.

\begin{figure*}[t]
\centering
\includegraphics[width=\textwidth]{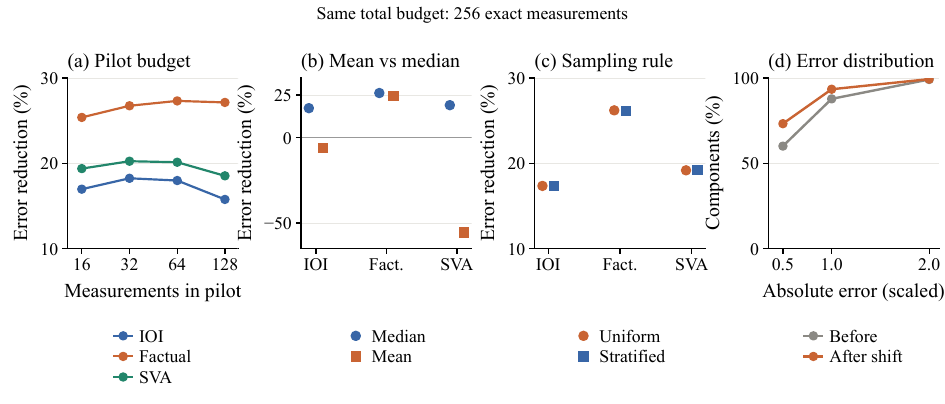}
\caption{Within-prompt \guided{} ablations on Gemma-2-9B at 256 measurements. Panels (a) to (c) report MAE reductions relative to verification alone. Panel (d) shows factual absolute-error distributions before and after the common shift, scaled by the original root mean square error.}
\label{fig:pilot-ablations-full}
\end{figure*}

\begin{table*}[t]
\centering

\footnotesize
\setlength{\tabcolsep}{4pt}
\begingroup
\begin{tabular}{llrrrr}
\toprule
Model & Task & Groups & \shortstack{\guided{} vs\\verification (\%)} & \shortstack{\guided{} vs\\verification + shift (\%)} & \shortstack{FP64: \guided{} vs\\verification + shift (\%)}\\
\midrule
Gemma-9B & IOI & 100 & 17.3 [14.7, 20.0] & 57.2 [53.9, 60.1] & 56.4 [53.0, 59.6]\\
Gemma-9B & Factual & 100 & 26.2 [23.4, 28.9] & 23.6 [17.4, 30.3] & 22.9 [16.4, 30.0]\\
Gemma-9B & SVA & 100 & 19.1 [16.7, 21.5] & 55.1 [48.0, 61.8] & 55.1 [47.9, 61.8]\\
Qwen3-14B & IOI & 52 & 24.0 [19.2, 28.5] & 45.8 [40.7, 50.3] & 45.4 [40.1, 50.1]\\
Qwen3-14B & Factual & 56 & 12.2 [9.6, 15.0] & 75.9 [71.0, 79.3] & 75.7 [70.7, 79.2]\\
Qwen3-14B & SVA & 53 & 10.9 [7.9, 14.1] & 75.3 [70.9, 78.8] & 75.2 [70.8, 78.6]\\
Qwen3-32B & IOI & 10 & 16.8 [10.8, 22.3] & 71.6 [60.2, 78.1] & 71.1 [59.6, 77.6]\\
Qwen3-32B & Factual & 31 & 17.5 [13.7, 21.3] & 70.1 [62.2, 75.5] & 70.1 [62.1, 75.4]\\
Qwen3-32B & SVA & 40 & 13.7 [10.6, 16.6] & 78.0 [73.4, 81.2] & 78.4 [73.9, 81.5]\\
\bottomrule
\end{tabular}
\endgroup
\caption{Within-prompt \guided{} MAE reductions at equal measurement counts. Intervals are descriptive paired-group 95\% intervals. FP64 scoring keeps the deployed FP32 estimates fixed.}\label{tab:pilot-controls-full}
\end{table*}

\begin{table}[t]
\centering

\footnotesize
\setlength{\tabcolsep}{3pt}
\begingroup
\begin{tabular}{lrrrr}
\toprule
Task & Pilot & \shortstack{Stratified\\median} & \shortstack{Uniform\\median} & \shortstack{Stratified\\mean}\\
\midrule
IOI & 16 & 16.7 & 16.5 & 7.0\\
IOI & 32 & 17.6 & 17.5 & 8.2\\
IOI & 64 & 17.3 & 17.3 & 5.0\\
IOI & 128 & 15.3 & 15.3 & 5.9\\
Factual & 16 & 24.2 & 24.1 & 23.9\\
Factual & 32 & 25.6 & 25.7 & 25.3\\
Factual & 64 & 26.2 & 26.2 & 25.6\\
Factual & 128 & 26.1 & 26.1 & 25.7\\
SVA & 16 & 18.1 & 17.7 & 1.8\\
SVA & 32 & 19.2 & 19.1 & 9.8\\
SVA & 64 & 19.2 & 19.2 & 10.4\\
SVA & 128 & 17.4 & 17.4 & 9.7\\
\bottomrule
\end{tabular}
\endgroup
\caption{Gemma \guided{} pilot ablations with replayed samples. Entries give percent lower MAE than verification alone at 256 total measurements.}\label{tab:pilot-design-full}
\end{table}

Figure~\ref{fig:pilot-ablations-full}a and Table~\ref{tab:pilot-size-acquired} use the acquired pilot-size cohort below. Panel (b) uses the measured pilot labels from 100 Gemma groups per task; panel (c) and Table~\ref{tab:pilot-design-full} replay samples on those groups. All use the layer/group aggregation above and descriptive uncertainty.

The acquired pilot-size study uses 20 groups per task, all three layers, four pilot sizes, and three sampling seeds. The same complete groups enter every size. Pilots of 32 or 64 measurements capture most of the benefit (Table~\ref{tab:pilot-size-acquired}).

\begin{table}[H]
\centering

\footnotesize
\setlength{\tabcolsep}{3pt}
\begin{tabular}{lrrrrr}
\toprule
Task & Groups & $P=16$ & $P=32$ & $P=64$ & $P=128$\\
\midrule
IOI & 20 & 17.0 & 18.2 & 18.0 & 15.8\\
Factual & 20 & 25.4 & 26.7 & 27.3 & 27.1\\
SVA & 20 & 19.4 & 20.2 & 20.1 & 18.5\\
\bottomrule
\end{tabular}
\caption{Acquired pilot-size ablation on Gemma-2-9B. Entries are MAE reductions relative to verification alone, in percent, at 256 total measurements. The same complete groups enter each pilot size.}\label{tab:pilot-size-acquired}
\end{table}

At $P=64$, 10,000 paired group-bootstrap draws give descriptive 95\% intervals of 11.4\% to 24.6\% (IOI), 22.2\% to 32.1\% (factual), and 14.5\% to 25.9\% (SVA). Group wins are 19/20, 20/20, and 19/20; the main results use the separate 100-group cohorts.

Figure~\ref{fig:pilot-ablations-full}d pools factual residuals on their original RMS scale. Fractions within 0.5, 1, and 2 RMS use exact histogram bin boundaries. The shifted distribution applies the pilot offset before exact replacements. At least three quarters of residuals share a strict sign in 43\%, 51\%, and 47\% of Gemma IOI, factual, and SVA prompt-layer cases, respectively, measured from their quartiles.

\subsection{Mechanism and Precision Checks}
\label{app:mechanism-precision}

After predictions and selected sets are fixed, reference residuals supply the diagnostic best common shift on the unmeasured set. Its realized correction benefit is $H_U-E_U$. $C_{\rm alloc}$ compares the method with the equal-count magnitude prefix $V_B$; $C_{\rm count}$ accounts for the runtime-matched prefix $V_{B_0}$. Table~\ref{tab:pilot-decomposition} normalizes terms by mean matched-baseline MAE.

\begin{table}[t]
\centering

\small
\setlength{\tabcolsep}{4pt}
\begin{tabular}{lrrrrr}
\toprule
Task & $H_U$ & $E_U$ & $C_{\rm alloc}$ & $C_{\rm count}$ & Gain\\
\midrule
IOI & 19.76 & 0.72 & 1.74 & $-$0.13 & 17.42\\
Factual & 27.37 & 0.83 & 0.36 & $-$0.02 & 26.21\\
SVA & 21.07 & 0.82 & 1.36 & $-$1.04 & 19.92\\
\bottomrule
\end{tabular}
\caption{Theorem~\ref{thm:within}'s decomposition on Gemma, in percentage points of matched-baseline MAE. Negative count cost means fewer baseline checks.}\label{tab:pilot-decomposition}
\end{table}

Primary acquisition discrepancy is $D_{\rm acq}=0$. The fitted median retains 96.4\%, 97.0\%, and 96.1\% of the best shift's benefit on IOI, factual completion, and SVA. On SVA, runtime matching contributes 1.04 percentage points and displaced verification costs 1.36; the equal-count gain is 19.1\%. Group-averaged wins are 96/100, 97/100, and 98/100.

FP64 readout checks retain the FP32 model body, hidden states, attribution vector, and runtime match. \emph{Fixed} scores deployed FP32 estimates against the FP64 reference; \emph{refit} also uses FP64 acquired values for the pilot shift and measured coordinates. Neither runs a full-FP64 model (Table~\ref{tab:single-prompt-precision}).

\begin{table}[t]
\centering

\small
\setlength{\tabcolsep}{3.5pt}
\begin{tabular}{llrrr}
\toprule
Model & Task & Primary & Fixed & Refit\\
\midrule
Gemma-9B & IOI & 17.4 & 10.1 & 16.3\\
 & Factual & 26.2 & 23.4 & 25.4\\
 & SVA & 19.9 & 15.4 & 20.4\\
Qwen3-14B & IOI & 24.0 & 21.0 & 24.5\\
 & Factual & 12.3 & 10.4 & 12.8\\
 & SVA & 10.9 & 4.1 & 9.5\\
Qwen3-32B & IOI & 17.3 & 13.6 & 16.9\\
 & Factual & 17.5 & 8.2 & 14.2\\
 & SVA & 13.7 & 8.0 & 14.2\\
\bottomrule
\end{tabular}
\caption{Percent lower MAE under the primary reference and two FP64 sensitivity checks. Fixed estimates and FP64-acquired refits keep the budget match.}\label{tab:single-prompt-precision}
\end{table}

Fixed-estimate 95\% intervals are 7.2\% to 13.1\%, 20.5\% to 26.3\%, and 12.6\% to 18.3\% on Gemma IOI, factual, and SVA. Qwen3-14B SVA is 4.1\% (0.5\% to 8.1\%). For 32B, fixed-estimate intervals are 8.1\% to 17.7\%, 3.2\% to 13.1\%, and 4.4\% to 11.8\%; primary intervals are 11.6\% to 22.7\%, 13.6\% to 21.2\%, and 10.7\% to 16.7\%. Primary group wins are 9/10, 30/31, and 39/40, on the same cohorts as Table~\ref{tab:single-prompt}.

\subsection{Fixed Budget across Component Populations}
\label{app:component-scaling}
Nested populations contain 1,024, 4,096, or all neurons in a predetermined random order. Each receives $P=64$, $B=256$ on the same groups, layers, and seeds as the full-population study, with its own runtime-matched verification prefix. Population sizes are dependent views of each prompt group (Table~\ref{tab:component-scaling}).

\begin{table}[t]
\centering

\small
\setlength{\tabcolsep}{3pt}
\begin{tabular}{llrrr}
\toprule
Model & Task & $N=1,024$ & $N=4,096$ & Full\\
\midrule
Gemma-9B & IOI & 22.7\% & 19.4\% & 17.4\%\\
 & Factual & 27.8\% & 26.8\% & 26.2\%\\
 & SVA & 22.6\% & 20.9\% & 19.9\%\\
Qwen3-14B & IOI & 27.4\% & 25.7\% & 24.0\%\\
 & Factual & 16.2\% & 13.9\% & 12.3\%\\
 & SVA & 19.2\% & 14.4\% & 10.9\%\\
Qwen3-32B & IOI & 23.7\% & 20.1\% & 17.3\%\\
 & Factual & 22.4\% & 19.6\% & 17.5\%\\
 & SVA & 19.6\% & 16.7\% & 13.7\%\\
\bottomrule
\end{tabular}
\caption{MAE reduction against runtime-matched verify-more at three predetermined component populations. Full populations are 14,336, 17,408, and 25,600 neurons for the three models.}\label{tab:component-scaling}
\end{table}

\section{Error-Structure Diagnostics}
\label{app:stability}
\paragraph{Probe noise grows with interactions.}
Define $\kappa_i=\sqrt{\sum_{j\ne i}G_{ij}^2}/\max(|G_{ii}|,10^{-30})$, the ratio of off-diagonal to diagonal curvature that sets the sketch noise in Equation~\ref{eq:variance}. Its median rises from 5.8 to 77.2 over $N=32$ to $1{,}152$ heads in Qwen3-8B IOI (Figure~\ref{fig:mechanism}). Separately, 41 $N/R$ settings on ten held-out IOI prompts give a median observed-to-predicted sketch-MSE ratio of 1.003 (range 0.945 to 1.042). Each point averages row-level MSE ratios, and settings are dependent.

\begin{figure}[t]
\centering
\includegraphics[width=7.7cm]{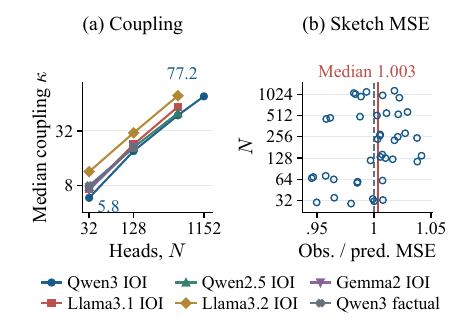}
\caption{Sketch-noise diagnostics. Left: median coupling across six model-task settings. Right: 41 $N/R$ means on Qwen3-8B IOI; dashed line denotes equality and red line the median. Points are dependent, not confidence intervals.}
\label{fig:mechanism}
\end{figure}

Table~\ref{tab:stability} compares fixed calibration halves. AP error correlations use heads measured in both halves; probe correlations use all heads. Top-16 overlap scores the model fitted on all 50 calibration groups separately on each half. These describe calibration stability rather than held-out allocation performance.

\begin{table}[t]
\centering

\small
\setlength{\tabcolsep}{4pt}
\begin{tabular}{llrrr}
\toprule
Model & Task & AP error $\rho$ & Probe $\rho$ & Top-16 kept\\
\midrule
Qwen3 & IOI & 0.69 & 0.97 & 16 / 16\\
Qwen3 & Factual & 0.80 & 0.97 & 15 / 16\\
Qwen3 & SVA & 0.57 & 0.95 & 14 / 16\\
Qwen2.5 & IOI & 0.69 & 0.95 & 13 / 14\\
Qwen2.5 & Factual & 0.81 & 0.97 & 15 / 15\\
Qwen2.5 & SVA & 0.62 & 0.95 & 12 / 15\\
\bottomrule
\end{tabular}
\caption{Stability between fixed calibration halves. AP error $\rho$ uses heads measured in both halves (87 to 137 per setting); probe $\rho$ uses all heads. Top-16 kept reports overlap for the \guided{} / \probecal{} model fitted on all 50 calibration groups, scored on each half.}\label{tab:stability}
\end{table}

For the dataset-average effects, the absolute mean signed error across heads is 81\% to 96\% of the mean absolute error. Equivalently, with $\bar g_i$ and $\bar\Delta_i$ denoting the estimated and reference mean effects, the diagnostic is $|\sum_i(\bar g_i-\bar\Delta_i)|/\sum_i|\bar g_i-\bar\Delta_i|$. It describes agreement in error sign across heads after averaging prompts.

\section{Confirmation Protocol}
\label{app:protocol}
Across-prompt evaluation uses 50 calibration and 100 evaluation groups in each of six model-task cells. All 1,152 Qwen3-8B heads (revision \texttt{b968826d}) and 784 Qwen2.5-7B-Instruct heads (\texttt{a09a3545}) are included \citep{yang2025qwen3,qwen2025qwen25}. Collection uses H100, PyTorch~2.9.1, Transformers~4.57.1, CUDA~12.6, FP32 eager attention, and disabled TF32. Groups shared across models remain linked. IOI equally mixes matched and template/entity-shifted groups; exact and reversed contrasts do not overlap across splits.

Arms comprise \atp{} ranking (verify-more), random ranking, \guided{}, a cost-matched guided prefix, \probecal{}, and a capacity-matched AP-only learner. The reference-error oracle is a diagnostic ceiling. The five nonbaseline arms, six cells, four budgets, and three horizons define 360 comparisons. \atpstar{} is evaluated within prompts (Appendix~\ref{app:single-prompt-comparators}); all across-prompt arms use \atp{} estimates.

\guided{} acquires 1,600 labels per cell by magnitude-guided round robin. \probecal{} uses eight Rademacher probes per calibration group, totaling 400 HVPs. Fits and the capacity-matched control follow Appendices~\ref{app:acquisition} and~\ref{app:probecal}.

The output metric is the final-position logit difference between correct and counterfactual tokens, relative to the clean baseline. Per-prompt MAE is computed after $B\in\{16,32,64,128\}$ verifications. Dataset-mean error averages effects over the 100 held-out groups before taking absolute error; it is not average per-prompt MAE. Recall uses the ten largest absolute reference effects.

Primary verification and scoring share a measured full AP table; separately acquired AP checks at $B=16$ and $B=32$, with their cost included, confirm this shared table. Total costs include calibration, fitting, transfer, model loading, attribution, and verification. Analysis lengths of 100, 1,000, and 10,000 prompts project measured costs and observed errors; they are not additional evaluation samples.

We use 20,000 paired group-bootstrap draws with the prescribed groups and strata. Two-sided Bonferroni percentile intervals use tail probability $0.05/(2\cdot360)$, about 1.39 draws per tail, and condition on the fitted calibration and measured costs.

\subsection{Across-Prompt Result Summaries}
\label{app:across-result-summaries}
Table~\ref{tab:allocation} reports allocation at matched projected total cost, including calibration. Table~\ref{tab:dataset} gives dataset-average error ratios against uncorrected \atp{} and calibration coverage.

\begin{table}[t]
\centering

\scriptsize
\setlength{\tabcolsep}{1.8pt}
\begin{tabular}{llcccr}
\toprule
& & \multicolumn{2}{c}{Error ratio} & \multicolumn{2}{c}{\guided{} labels}\\
\cmidrule(lr){3-4}\cmidrule(lr){5-6}
Model & Task & \guided{} & \probecal{} & Heads & Coverage\\
\midrule
Qwen3 & IOI & 0.95 & \textbf{1.49} & 121/1,152 & 10.5\%\\
Qwen3 & Factual & \textbf{3.98} & \textbf{1.87} & 145/1,152 & 12.6\%\\
Qwen3 & SVA & 1.23 & 1.14 & 265/1,152 & 23.0\%\\
Qwen2.5 & IOI & \textit{0.66} & \textbf{1.09} & 163/784 & 20.8\%\\
Qwen2.5 & Factual & \textbf{4.78} & \textbf{1.86} & 141/784 & 18.0\%\\
Qwen2.5 & SVA & 1.34 & 0.99 & 217/784 & 27.7\%\\
\bottomrule
\end{tabular}
\caption{Dataset-average correction. Error ratio is uncorrected divided by corrected MAE. Bold and italic mark adjusted significant improvements and harms. Heads and coverage refer to \guided{} calibration.}\label{tab:dataset}
\end{table}

\section{Cost Frontiers for All Cells}
\label{app:frontiers}
Figures~\ref{fig:frontier}, \ref{fig:frontiers-a}, and~\ref{fig:frontiers-b} cover all six across-prompt cells, budgets, and projected analysis lengths.

\begin{figure*}[t]
\centering
\includegraphics[width=\textwidth]{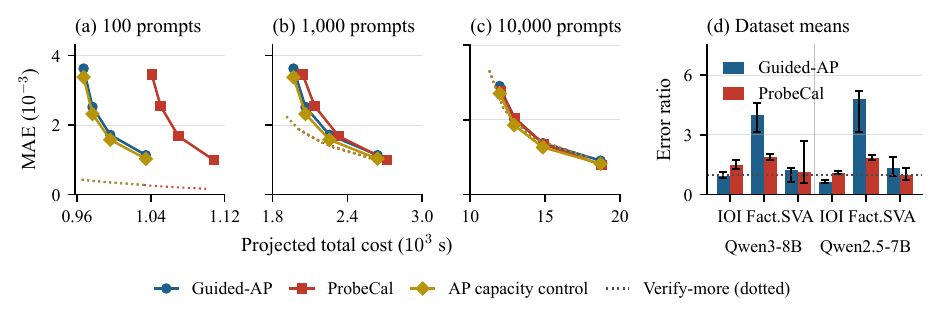}
\caption{Across-prompt evaluation on Qwen3 IOI. Error against projected total cost at three analysis lengths, with markers for verification budgets and dotted curves for verify-more, followed by the dataset-mean correction ratios for all six cells with adjusted intervals.}
\label{fig:frontier}
\end{figure*}

Calibration cost favors verify-more at 100 prompts throughout the grid. At 1,000 prompts, 108 of 120 adjusted comparisons favor verify-more and none favors learning. At 10,000 prompts, the useful cells and budgets depend on the signal. \guided{} improves Qwen3 IOI and both SVA cells at $B=16$; \probecal{} also improves factual allocation at $B\ge64$. The figures retain the remaining comparisons and controls.

Lower MAE need not improve recall. One Qwen3 IOI \probecal{} setting reduces MAE by 11\% while top-ten recall falls from 99.2\% to 97.0\%.

Projected break-even ranges from 1,900 to 24,000 prompts at $B=16$ and 900 to 15,500 at $B=128$. Recall intervals usually include no change; some IOI $B=16$ settings reduce recall by 0.03 to 0.07.

\begin{figure*}[tp]
\centering
\includegraphics[width=\textwidth]{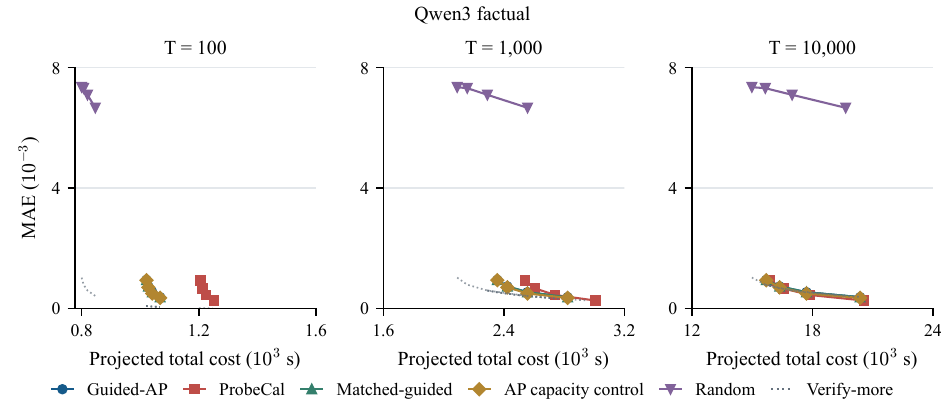}\\[2pt]
\includegraphics[width=\textwidth]{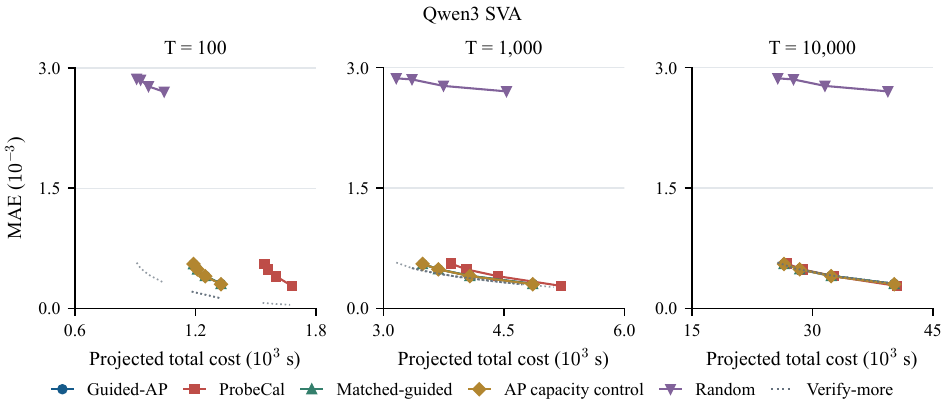}\\[2pt]
\includegraphics[width=\textwidth]{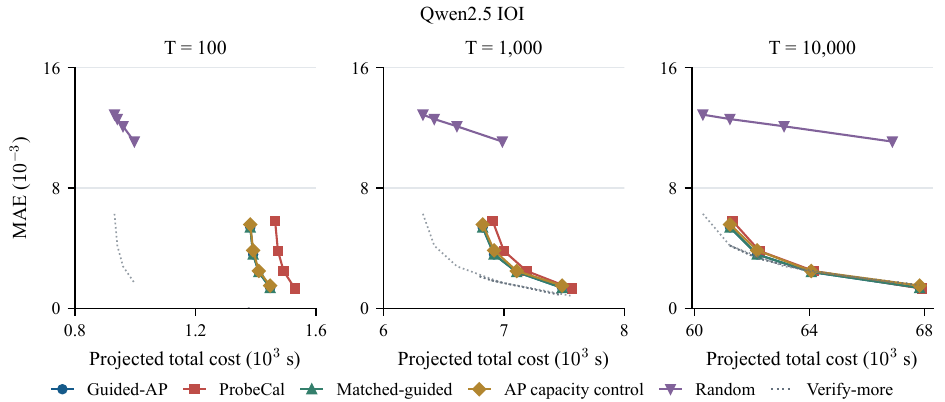}
\caption{Cost-error frontiers for three more cells at 100, 1,000, and 10,000 prompts. Conventions follow Figure~\ref{fig:frontier}.}
\label{fig:frontiers-a}
\end{figure*}

\begin{figure*}[!t]
\centering
\includegraphics[width=\textwidth]{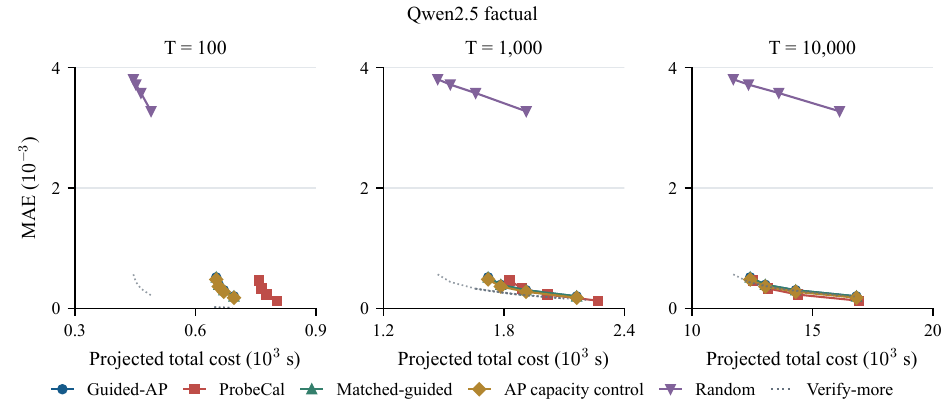}\\[2pt]
\includegraphics[width=\textwidth]{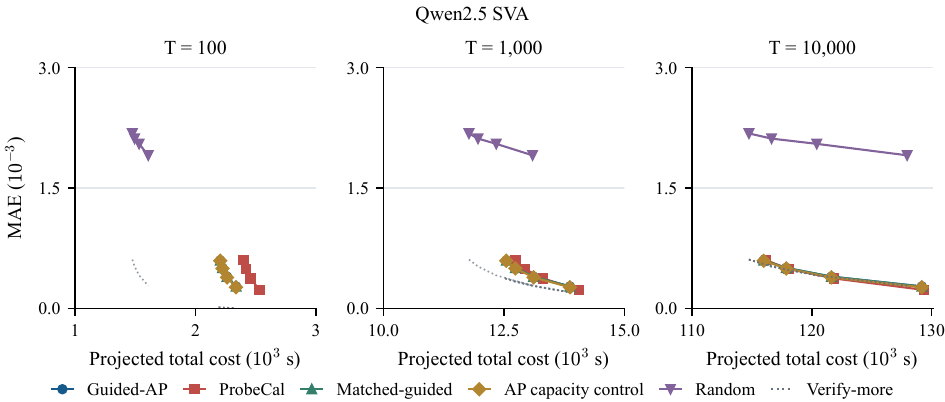}
\caption{Cost-error frontiers for the remaining two cells. Conventions as in Figure~\ref{fig:frontier}.}
\label{fig:frontiers-b}
\end{figure*}

\section{Additional Controls on the Confirmation Cells}
\label{app:controls}
These post hoc controls use the same six cells and 20,000 paired bootstrap draws, with their analysis protocol fixed before computation.

\paragraph{Dataset correction at matched total cost.}
Verify-more uses the largest affordable magnitude-ranked AP prefix. It wins in all settings at 100 prompts and in ten of twelve at 1,000 prompts. At 10,000 prompts it affords no additional checks, so the comparison is with uncorrected \atp{}. Where correction improves that baseline, projected break-even is 527 to 2,744 prompts. The IOI exceptions appear in Table~\ref{tab:dataset}.

\paragraph{Learning residuals versus averaging measurements.}
Using the same 1,600 queries directly as component-average effects produces 5.8 to 14.9 times more error than residual correction (Equation~\ref{eq:dataset-mean-correction}) in all six cells, with every adjusted interval above one, and is worse than uncorrected \atp{}.

\paragraph{Independently acquired verification.}
Separate AP forwards at $B=16$ and $B=32$ reproduce the primary values exactly in all six cells and six arms; all 75 testable prespecified comparisons retain their direction.

\section{Cost Accounting and Numerical Scope}
\label{app:cost}
For method $m$ under a fixed loading and batching schedule, projected total cost is
\begin{equation}
C_m(T,B)=C_m^{\rm init}+T c_m(B),
\label{eq:totalcost}
\end{equation}
where $C_m^{\rm init}$ includes calibration, fitting, and the specified initialization costs. The per-prompt term $c_m(B)$ includes attribution, ranking, activation handling, transfer, and verification. $T$ is a projected analysis length, not the number of evaluation prompts over which dataset averages are measured. The prespecified analysis lengths are 100, 1,000, and 10,000 prompts, and the allocation budgets are $B\in\{16,32,64,128\}$.

For an across-prompt policy with $B$ verifications, the matched verify-more budget is
\begin{equation*}
\begin{aligned}
b_m(T,B)&=\max_{b\in\{0,\ldots,N\}}b\\
&\text{subject to } C_0(T,b)\leq C_m(T,B),
\end{aligned}
\end{equation*}
Here $C_0$ includes attribution and magnitude-ranked verification. Measured batch padding and overhead determine the integer budget; unused compute is retained. Within-prompt matching uses measured runtime without prior calibration. Dense curves are descriptive; adjusted inference uses the prespecified budgets and horizons.

Equation~\ref{eq:totalcost} assumes one initialization. Break-even horizons project measured cost components.

On Qwen3 factual completion, 1,600 AP labels at batch size 16 took 13.98 seconds (0.009 seconds/check); 400 serial HVPs took 51.77 seconds (0.13 seconds/HVP). At these rates, one HVP costs about 15 exact checks. Full calibration costs, including probes, fitting, transfer, and model loading, range from 384 to 1,050 seconds per cell and are included in the comparisons.

Deployment includes initialization, calibration, fitting, model loading, gradients, activation handling, transfer, and verification; reference collection is separate. Rates are specific to the measured hardware configuration.

Interventions add directions to live activations, preserving upstream propagation. Derivatives and references use FP32 eager execution; same-hidden-state FP64 readouts assess numerical sensitivity.

\end{document}